\documentclass[11pt]{article}
\usepackage{amsfonts}
\usepackage{mathrsfs}
\usepackage{amsmath,amssymb}
\usepackage{mathtools}

\DeclareMathOperator{\EX}{\mathbb{E}}

\usepackage[latin1]{inputenc}
\usepackage{amsmath,amssymb}

\usepackage{amsthm}
\usepackage{latexsym}
\usepackage{epstopdf}
\usepackage{geometry}  
\usepackage{pict2e,picture}
\usepackage{bigints}  
\usepackage[thinlines]{easytable}
\usepackage{enumerate}
\usepackage{array}
\usepackage{fix-cm}
\usepackage{tikz}
\usetikzlibrary{matrix}

\newcolumntype{C}[1]{>{\centering\arraybackslash}m{#1}}
\newcolumntype{R}[1]{>{\raggedleft\arraybackslash}m{#1}}
 \usepackage{lipsum}
 
 \usepackage[symbol]{footmisc}

\usepackage[active]{srcltx}
\usepackage{graphicx}
\usepackage{epstopdf}
\usepackage{enumitem}   
\usepackage[T1]{fontenc}
\usepackage{amsmath}

\usepackage{footnote}
\usepackage{footmisc}
\makesavenoteenv{tabular}
\makesavenoteenv{table}

\usepackage{bbm}

\DeclareMathOperator*{\argmin}{arg\,min}

\usepackage[
bookmarks=true,         
bookmarksnumbered=true, 
colorlinks=true, pdfstartview=FitV, linkcolor=blue, citecolor=blue,
urlcolor=blue]{hyperref}

\newtheorem{thm}{Theorem}[section]

\newtheorem{lemma}{Lemma}[section]

\newtheorem{remark}{Remark}[section]

\renewcommand\footnotemark{}

\def\qed{{\hfill $\square$ \bigskip}}

\usepackage{dsfont}

\begin{document}

\date{\vspace{-14ex}}
\title{Function approximation by deep neural networks with parameters $\{0,\pm \frac{1}{2}, \pm 1, 2\}$}\maketitle

\begin{center}
	\bigskip Aleksandr Beknazaryan \footnote{a.beknazaryan@utwente.nl}
	
    \textit{University of Twente}
	\bigskip
	
\end{center}

\begin{abstract}In this paper it is shown that $C_\beta$-smooth functions can be approximated by deep neural networks with ReLU activation function and with parameters $\{0,\pm \frac{1}{2}, \pm 1, 2\}$. The $l_0$ and $l_1$ parameter norms of considered networks are thus equivalent. The depth, width and the number of active parameters of the constructed networks have, up to a logarithmic factor, the same dependence on the approximation error as the networks with parameters in $[-1,1]$. In particular, this means that the nonparametric regression estimation with the constructed networks attains the same convergence rate as with sparse networks with parameters in $[-1,1]$. 
\end{abstract}
\begin{center}
\end{center}

\section{Introduction}
 The problem of function approximation with neural networks has been of big interest in mathmatical research for the last several decades. Various results have been obtained that describe the approximation rates in terms of the structures of the networks and the properties of the approximated functions. One of the most remarkable results in this direction is the universal approximation theorem, which shows that even shallow (but sufficiently wide) networks can approximate continuous functions arbitrarily well (see \cite{Sc} for the overview and possible proofs of the theorem). Also, in \cite{L} it was shown that integrable functions can be approximated by networks with fixed width. Those networks, however, may need to be very deep to attain small approximation errors.  Yet, from a pragmatic point of view, and, in particular, in statistical applications, allowing very big number of network parameters may be impractical. The reason is that in this case controlling the complexity of approximant networks at an optimal rate becomes problematic. Complexities of classes of neural networks are usually described in terms of their covering numbers and entropies. Those two concepts also 
 play an important role in various branches of statistics, such as regression analysis, density estimation and empirical processes (see, e.g., \cite{G}, \cite{Bi}, \cite{P}). In particular, in regression estimation the following dichotomy usually comes up while selecting the class of functions from which the estimator will be chosen: on the one hand, the selected class of approximants should be ``big'' enough to be able to approximate various non-trivial functions and on the other hand it should have ``small'' enough entropy to attain good learning rates. Thus, the general problem is to obtain powerful classes of functions with well controllable entropies. As to the powerfulness of classes of neural networks, it has recently been shown (\cite{SH}, \cite{Y}) that with properly chosen architecture the classes of sparse deep neural networks with ReLU activation function can well approximate smooth functions. In particular, it is shown in \cite{Y}, that $C_\beta$-smooth functions on $[0,1]^d$ can be $\varepsilon$-approximated by deep ReLU networks with $O(\varepsilon^{-d/\beta}\log_2(1/\varepsilon))$ active (nonzero) parameters. A similar result for sparse ReLU networks with parameters in $[-1,1]$ has been obtained in \cite{SH}. The number of active parameters $s$ in those networks is much smaller than the total number of network parameters and the network depth $L$ depends logarithmically on the approximation error. Boundedness of parameters of the networks constructed in \cite{SH} implies that the $\varepsilon-$entropy of the approximating networks has order $O(sL^2\log_2(1/\varepsilon)).$ The main advantages of this entropy bound are its logarithmic dependence on $1/\varepsilon$, which allows to take the covering radius  $\varepsilon$ to be very small in applications, and its linear dependence on the sparisty $s$ and quadratic dependence on the depth $L$, both of which, as described above, can also be taken to be small. Using this entropy bound, it is then shown in \cite{SH}, that if the regression function is a composition of H\"older smooth functions, then sparse neural networks with depth $L\lesssim \log_2n$ and the number of active parameters $s\sim n^{\frac{t}{2\beta+t}}\log_2n,$ where $\beta>0$ and $t\geq 1$ depend on the structure and the smoothness of the regression function, attain the minimax optimal prediction error rate $n^{\frac{-2\beta}{2\beta+t}}$  (up to a logarithmic factor), where $n$ is the sample size. It would therefore be desirable to obtain a similar entropy bound for the spaces of networks for which the above $l_0$ (sparsity) regularization is replaced by the better practically implementable $l_1$ regularization.
 
Networks with $l_1$ norm of all parameters bounded by $1$ are considered in \cite{TXL}. As in those networks there are at most $1/\varepsilon^2$ parameters outside of the interval $(-\varepsilon^2,\varepsilon^2)$, an entropy bound of order $O((2/L)^{2L-1}/\varepsilon^2),$ has been obtained by taking in the covering networks the remaining parameters to be 0. This bound, however, depends polynomially on $1/\varepsilon$, and it leads to the convergence rate of order $1/\sqrt{n}$ for regression estimation with given $n$ samples. As it is discussed in \cite{TXL}, the rate $1/\sqrt{n}$ is seemingly the best possible for $l_1$ regularized estimators. Alternative approaches of sparsifying neural networks using derivatives, iterative prunings and clipped $l_1$ penalties are given in \cite{GEH}, \cite{HPTD}, \cite{HS} and \cite{OK}. 
 
To combine the advantages of both $l_0$ and $l_1$ regularizations, as well as to make the networks easier to encode, we consider networks with parameters $\{0,\pm\frac{1}{2}, \pm 1, 2\}$. The $l_0$ and $l_1$ parameter regularizations of those networks can differ at most by a factor of 2, which, in particular, allows to employ all the features induced from the sparsity of networks (including their entropy bounds) while imposing  $l_1$ constraints on their parameters. Moreover, discretization of parameters allows to calculate the exact number of networks (the 0--entropy) required to attain a given approximation rate. Importantly, the depth, the width and the number of active parameters in the approximant networks are equivalent to those of networks constructed in \cite{SH}. Hence, for the considered networks the $l_0$ parameter regularization can be replaced by the $l_1$ parameter regularization, leading, up to a logarithmic factor, to the same statistical guarantees as in \cite{SH}. In our construction the parameters $\pm 1$ are used to add/subtract the nodes, change, if necessary, their signs and transfer them to the next layers. The parameters $\pm\frac{1}{2}$ and $2$ are used to attain the values of the form $k/2^j\in[-1,1]$, $j\in\mathbb{N}$, $k\in\mathbb{Z},$ which can get sufficiently close to any number from $[-1,1]$. Note that this can also be done using only the parameters $\pm\frac{1}{2}$ and $1$. The latter, however, would require a larger depth and a bigger number of active nodes.
 
 \textit{Notation.} The notation $|\textbf{v}|_\infty$ is used for the $l_\infty$  norm of a vector $\textbf{v}\in\mathbb{R}^d$ and  $\|f\|_{L^\infty[0,1]^d}$ denotes the sup norm of a function $f$ defined on $[0,1]^d$, $d\in\mathbb{N}$. For $x,y\in\mathbb{R}$ we denote $x\lor y:=\max\{x,y\}$ and $(x)_+:=\max\{0,x\}$.  Also, to make them multiplicable with preceeding matrices, the vectors from $\mathbb{R}^d$, depending on the context, are considered as matrices from $\mathbb{R}^{d\times 1}$ rather than $\mathbb{R}^{1\times d}$. 

 \section{Main result}
Consider the set of neural networks with $L$ hidden layers and with ReLU activation function  $\sigma(x)=0\lor x=(x)_+$ defined by
$$\mathcal{F}(L,\textbf{p}):=\{f:[0,1]^d\to\mathbb{R}^{p_{L+1}}\; |\;\; f(\textbf{x})=W_L\sigma_{\textbf{v}_L} W_{L-1}\sigma_{\textbf{v}_{L-1}}...W_1\sigma_{\textbf{v}_1}W_0\textbf{x}\},$$
where $W_i\in\mathbb{R}^{p_i\times p_{i+1}}$ are weight matrices, $i=0,...,L,$ $\textbf{v}_i$ are shift vectors, $i=1,...,L,$ and $\textbf{p}=(p_0, p_1,...,p_{L+1})$ is the width vector with $p_0=d$. For a given shift vector $\textbf{v}=(v_1,...,v_p)$ and a given input vector $\textbf{y}=(y_1,...,y_p)$ the action of shifted activation function $\sigma_{\textbf{v}}$ on $\textbf{y}^\top$ is defined as
\[\sigma_{\textbf{v}}(\textbf{y}^\top)= \begin{pmatrix}
\sigma(y_1-v_1),  \cdot\cdot\cdot,  \sigma(y_p-v_p)
\end{pmatrix}^\top.\]
It is assumed that network parameters (the entries of matrices $W_i$ and shift vectors $\textbf{v}_i$) are all in $[-1,1]$. For $s\in\mathbb{N}$ let $\mathcal{F}(L,\textbf{p},s)$ be the subset of $\mathcal{F}(L,\textbf{p})$ consisting of networks with at most $s$ nonzero parameters. In \cite{SH}, Theorem 5, the following approximation of $\beta$-H\"older functions belonging to the ball 
\begin{align*}
\mathcal{C}^\beta_d(K)=\bigg\{f:[0,1]^d\to\mathbb{R}: \sum\limits_{0\leq|\boldsymbol{\alpha}|<\beta}\|\partial^{\boldsymbol{\alpha}}f\|_{L^\infty[0,1]^d}+\sum\limits_{|\boldsymbol{\alpha}|=\lfloor\beta\rfloor}\sup\limits_{\substack{\textbf{x},\textbf{y}\in[0,1]^d \\ \textbf{x}\neq \textbf{y}}}\frac{|\partial^{\boldsymbol{\alpha}}f(\textbf{x})-\partial^{\boldsymbol{\alpha}}f(\textbf{y})|}{|\textbf{x}-\textbf{y}|_\infty^{\beta-\lfloor\beta\rfloor}}\leq K\bigg\}
\end{align*}
with networks from $\mathcal{F}(L,\textbf{p},s)$ is given:
\begin{thm}\label{or}
For any function $f\in\mathcal{C}^\beta_d(K)$ and any integers $m\geq 1$ and $N\geq(\beta+1)^d\lor(K+1)e^d,$ there exists a network $\tilde{f}\in\mathcal{F}(L,\emph{\textbf{p}},s)$ with depth
$$L=8+(m+5)(1+\lceil\log_2(d\lor\beta)\rceil),$$
width $$|\emph{\textbf{p}}|_\infty=6(d+\lceil\beta\rceil)N$$	
and number of nonzero parameters
$$s\leq141(d+\beta+1)^{3+d}N(m+6),$$
such that 
$$\|\tilde{f}-f\|_{L^\infty[0,1]^d}\leq(2K+1)(1+d^2+\beta^2)6^dN2^{-m}+K3^\beta N^{-\frac{\beta}{d}}.$$
\end{thm}
The proof of the theorem is based on local sparse neural network approximation of  Taylor polynomials of the function $f$. 

Our goal is to attain an identical approximation rate for networks with parameters in $\{0,\pm \frac{1}{2}, \pm 1, 2\}$. In our construction we will omit the shift vectors (by adding a coordinate $1$ to the input vector \textbf{x}) and will consider the networks of the form
\begin{equation}\label{def}
\bigg\{f:[0,1]^d\to\mathbb{R}\; |\;\; f(\textbf{x})=W_L\circ\sigma\circ W_{L-1}\circ\sigma\circ...\circ\sigma\circ W_0 (1, \textbf{x})\bigg\}
\end{equation}
with weight matrices $W_i\in\mathbb{R}^{p_i\times p_{i+1}}$, $i=0,...,L,$ and with width vector $\textbf{p}=(p_0, p_1,...,p_{L+1}),$ $p_0=d$. In this case the ReLU activation function $\sigma(x)$ acts coordinate-wise on the input vectors.  Let $\mathcal{\widetilde{F}}(L,\textbf{p})$ be the set of networks of the form \eqref{def} with parameters in $\{0,\pm \frac{1}{2}, \pm 1, 2\}$. For $s\in\mathbb{N}$ let $\mathcal{\widetilde{F}}(L,\textbf{p}, s)$ be the subset of $\mathcal{\widetilde{F}}(L,\textbf{p})$ with at most $s$ nonzero parameters. We then have the following 
\begin{thm}\label{main}
	For any function $f\in\mathcal{C}^\beta_d(K)$ and any integers $m\geq 1$ and $N\geq(\beta+1)^d\lor(K+1)e^d,$ there exists a network $\tilde{f}\in\mathcal{\widetilde{F}}(\tilde{L},\normalfont\tilde{\textbf{p}},\tilde{s})$ with depth
$$\tilde{L}\leq4{\Delta}+2L,$$
	width $$|\normalfont\tilde{\textbf{p}}|_\infty\leq2(1+d+R+\Delta)\lor2^d|{\textbf{p}}|_\infty$$	
	and number of nonzero parameters
$$\tilde{s}\leq (1+d+R+\Delta)\tilde{L}+2^ds,$$
	such that 
	$$\|\tilde{f}-f\|_{L^\infty[0,1]^d}\leq(2K+1)(1+d^2+\beta^2)12^dN2^{-m}+(K+1)3^\beta N^{-\frac{\beta}{d}},$$
	where $\Delta\leq2\log_2(N^{\beta+d}Ke^d)$, $R\leq (2\beta)^dN$ and $L, \normalfont\textbf{p}$ and $s$ are the same as in Theorem \ref{or}.
\end{thm}
Let us now compare the above two theorems. First, the approximation errors in those theorems differ by a constant factor depending only on the input dimension $d$ (note that the values of $\beta, d$ and $K$ are assumed to be fixed). The depths and the number of nonzero parameters of the networks presented in Theorems \ref{or} and \ref{main} differ at most by $\log_2N$ multiplied by a constant depending on  $\beta, d$ and $K$, and the maximal widths of those networks differ at most by a constant factor $C(\beta, d, K)$. Thus, the architecture and the number of active parameters of network given in Theorem \ref{main} have, up to a logarithimc factor, the same dependence on the approximation error as the network given in Theorem \ref{or}.

\textbf{Application to nonparametric regression}. Consider a nonparametric regression model 
$$Y_i=f_0(\textbf{X}_i)+\epsilon_i,$$
where $f_0:[0,1]^d\to[-F, F]$ is the unknown regression function that needs to be recovered from $n$ observed iid pairs $(\textbf{X}_i, Y_i)$, $i=1,...,n$. The standard normal noise variables $\epsilon_i$ are assumed to be independent of $\textbf{X}_i$. For a set of functions $\mathcal{F}$ from $[0,1]^d$ to $[-F, F]$ and for an estimator $\hat{f}\in\mathcal{F}$ of $f_0$ define
$$\Delta_n=\Delta_n(\hat{f},f_0, \mathcal{F} )=\EX_{f_0}\bigg[\frac{1}{n}\sum_{i=1}^{n}(Y_i-\hat{f}(\textbf{X}_i))^2-\inf\limits_{f\in\mathcal{F}}\frac{1}{n}\sum_{i=1}^{n}(Y_i-f(\textbf{X}_i))^2\bigg].$$
The subscript $f_0$ indicates that the expectation is taken over the training data generated by our regression model and $\Delta_n(\hat{f},f_0, \mathcal{F})$ measures how close the estimator $\hat{f}$ is to the empirical risk minimizer.
Let also
$$R(\hat{f},f_0)=\EX_{f_0}[(\hat{f}(\textbf{X})-f_0(\textbf{X}))^2]$$
be the prediction error of the estimator $\hat{f},$ where $\textbf{X}\stackrel{\mathcal{D}}{=}\textbf{X}_1$ is independent of the sample $(\textbf{X}_i, Y_i)$. The following oracle-type inequality is obtained in \cite{SH}, Lemma 4:
\begin{lemma}\label{Oracle}
	For any $\delta\in(0,1]$ 
	$$R(\hat{f},f_0)\leq4\bigg[\inf\limits_{f\in\mathcal{F}}\EX[\normalfont(f(\textbf{X})-f_0(\textbf{X}))^2]+F^2\frac{18\log_2\mathcal{N}(\delta,\mathcal{F},\|\cdot\|_\infty)+72}{n}+32\delta F+\Delta_n\bigg],$$
	where $\mathcal{N}(\delta,\mathcal{F},\|\cdot\|_\infty)$ is the covering number of $\mathcal{F}$ of radius $\delta$ taken with respect to the $\|\cdot\|_\infty$ distance of functions on $[0,1]^d$.
\end{lemma}
Assume that $f_0\in\mathcal{C}^\beta_d(K)$ with $F\geq\max(K,1)$.
Taking in Theorem \ref{main} $r=d$, $m=\lceil \log_2n\rceil$ and $N=n^{\frac{d}{2\beta+d}},$ we get the existence of a network $\tilde{f}_n\in\mathcal{\widetilde{F}}(\tilde{L}_n,\normalfont\tilde{\textbf{p}}_n,\tilde{s}_n)$ with $\tilde{L}_n\leq c\log_2n$, $|\normalfont\tilde{\textbf{p}}_n|_\infty\leq cn^{\frac{d}{2\beta+d}}$ and $\tilde{s}_n\leq cn^{\frac{d}{2\beta+d}}\log_2n$
such that
\begin{equation}\label{ineq}
\|\tilde{f}_n-f_0\|^2_{L^\infty[0,1]^d}\leq cn^{\frac{-2\beta}{2\beta+d}},
\end{equation}
where $c=c(\beta, d, F)$ is some constant. In order to apply Lemma \ref{Oracle} it remains to estimate the covering number  $\mathcal{N}(\delta,\mathcal{\widetilde{F}}(\tilde{L}_n,\normalfont\tilde{\textbf{p}}_n,\tilde{s}_n),\|\cdot\|_\infty)$. Note however, that since the parameters of networks from $\mathcal{\widetilde{F}}(\tilde{L}_n,\normalfont\tilde{\textbf{p}}_n,\tilde{s}_n)$ belong to the discrete set $\{0,\pm \frac{1}{2}, \pm 1, 2\},$ we can calculate the exact number of networks from $\mathcal{\widetilde{F}}(\tilde{L}_n,\normalfont\tilde{\textbf{p}}_n,\tilde{s}_n)$, or, in other words, we can upper bound the covering number of radius $\delta=0$. Indeed, as there are at most $(\tilde{L}_n+1)|\normalfont\tilde{\textbf{p}}_n|_\infty^2$ parameters in the networks from $\mathcal{\widetilde{F}}(\tilde{L}_n,\normalfont\tilde{\textbf{p}}_n,\tilde{s}_n),$ then for a given $s$ there are at most $\bigg((\tilde{L}_n+1)|\normalfont\tilde{\textbf{p}}_n|_\infty^2\bigg)^s$ ways to choose $s$ nonzero parameters. As the nonzero parameters can take one of the 5 values $\{\pm \frac{1}{2}, \pm 1, 2\},$ then the total number of networks from $\mathcal{\widetilde{F}}(\tilde{L}_n,\normalfont\tilde{\textbf{p}}_n,\tilde{s}_n)$ is bounded by 
$$\sum\limits_{s\leq \tilde{s}_n}\bigg(5(\tilde{L}_n+1)|\normalfont\tilde{\textbf{p}}_n|_\infty^2\bigg)^s\leq\bigg(5(\tilde{L}_n+1)|\normalfont\tilde{\textbf{p}}_n|_\infty^2\bigg)^{\tilde{s}_n+1}.$$
Together with \eqref{ineq} and Lemma \ref{Oracle}, for the empirical risk minimizer 
$$\hat{f}_n\in\argmin_{f\in\mathcal{\widetilde{F}}(\tilde{L}_n,\normalfont\tilde{\textbf{p}}_n,\tilde{s}_n)}\sum_{i=1}^{n}(Y_i-f(\textbf{X}_i))^2$$ we get an existence of a constant $C=C(\beta, d, F)$ such that 

\begin{equation}\label{err}
R(\hat{f}_n,f_0)\leq Cn^{\frac{-2\beta}{2\beta+d}}\log_2^2n
\end{equation}
which coincides, up to a logarithmic factor, with the minimax estimation rate $n^{\frac{-2\beta}{2\beta+d}}$ of the prediction error for $\beta$-smooth functions.

\begin{remark}
	As the parameters of networks from $\mathcal{\widetilde{F}}(\tilde{L}_n,\normalfont\tilde{\textbf{p}}_n,\tilde{s}_n)$ belong to $\{0,\pm \frac{1}{2}, \pm 1, 2\}$, then, instead of defining the sparsity constraint $\tilde{s}_n$ to be the maximal number of nonzero parameters, we could define $\tilde{s}_n$ to be the upper bound of $l_1$ norm of all parameters of networks from $\mathcal{\widetilde{F}}(\tilde{L}_n,\normalfont\tilde{\textbf{p}}_n,\tilde{s}_n)$. As the $l_0$ and $l_1$ parameter norms of networks from $\mathcal{\widetilde{F}}(\tilde{L}_n,\normalfont\tilde{\textbf{p}}_n)$ can differ at most by a factor of 2, then this change of notation would lead to the same convergence rate as in \eqref{err}.
\end{remark}

\section{Proofs}
One of the ways to approximate functions by neural networks is based on the neural network approximation of local Taylor polynomials of those functions (see, e.g., \cite{SH}, \cite{Y}). Thus, in this procedure, approximation of the product $xy$ given the input $(x,y)$ becomes crucial. The latter is usually done by representing the product $xy$ as a linear combination of functions that can be approximated by neural network-implementable functions. For example, the approximation algorithm presented in \cite{SH} is based on the approximation of a function $g(x)=x(1-x)$, which then leads to an approximation of the product \begin{equation}\label{x}
xy=g\bigg(\frac{x-y+1}{2}\bigg)-g\bigg(\frac{x+y}{2}\bigg)+\frac{x+y}{2}-\frac{1}{4}.
\end{equation} The key observation is that the function $g(x)$ can be approximated by combinations of triangle waves and the latter can be easily implemented by neural networks with ReLU activation function. In the proof of Theorem \ref{or}, neural network approximation of function $(x,y)\mapsto xy$ is followed by approximation of the product $(x_1,...,x_r)\mapsto\prod_{j=1}^{r}x_j$ which then  leads to approximation of monomials of degree up to $\beta$. The result then follows by local approximation of Taylor polynomials of $f$. Below we show that all those approximations can also be performed using only the parameters $\{0,\pm \frac{1}{2}, \pm 1, 2\}$.

\begin{lemma}\label{xy} For any positive integer $m$, there exists a network 
	\emph{Mult}$_m\in\mathcal{\widetilde{F}}(2m+4, \emph{\textbf{p}})$, with $p_0=3,$ $p_{L+1}=1$ and $|\emph{\textbf{p}}|_\infty=9,$ such that
	
\begin{equation}\label{in}\normalfont |\textrm{{Mult}}_m(1, x, y)-xy|\leq 2^{-m}, \quad \textrm{for all} \; x,y\in[0,1].\end{equation}
\end{lemma}

\begin{proof} Consider the functions 
	$T^k:[0,2^{2-2k}]\to[0,2^{-2k}], k\in\mathbb{N},$ defined by
	\begin{equation}\label{T}
	T^k(x):=(x/2)_+-(x-2^{1-2k})_+=T_+(x)-T_-^k(x),
	\end{equation}
	where $T_+(x):=(x/2)_+$ and $T_-^k(x):=(x-2^{1-2k})_+$. In \cite{SH1}, Lemma A.1, it is shown that for the functions
	$R^k:[0,1]\to[0,2^{-2k}],$
	\begin{equation}\label{R}
	R^k=T^k\circ T^{k-1}\circ...\circ T^1,
	\end{equation} and for any positive integer $m$,
	\begin{equation}\label{g}
	|g(x)-\sum_{k=1}^mR^k(x)|\leq2^{-m}, \quad x\in[0,1],
	\end{equation}
	where $g(x)=x(1-x)$.
Taking into account \eqref{x} and \eqref{g}, we need to construct a network that computes
\begin{equation}\label{mul}
(1, x, y)\mapsto \bigg(\sum_{k=1}^{m+1}R^k\bigg(\frac{x-y+1}{2}\bigg)-\sum_{k=1}^{m+1}R^k\bigg(\frac{x+y}{2}\bigg)+\frac{x+y}{2}-\frac{1}{4}\bigg)_+\wedge 1.\end{equation}
Let us first construct a network $N_m$ with depth $2m$, width $4$ and weights $\{0,\pm \frac{1}{2}, \pm 1\}$ that computes
$$(1/4, T_+(u), h(u), T_-^1(u))\mapsto\sum_{k=1}^{m+1}R^k(u)+h(u),\quad u\in[0,1].$$
For this goal, we modify the network presented in \cite{SH1}, Fig. 2, to assure that the parameters are all in $\{0,\pm \frac{1}{2}, \pm 1\}$. More explicitly, denote
	\[A:= \begin{pmatrix}
\frac{1}{2} & 0 & 0 & 0
\\
0 & \frac{1}{2} & 0 & -\frac{1}{2}
\\
0 & 1 & 1 & -1
\\
-\frac{1}{2} & 1 & 0 & -1
\end{pmatrix}
\]
and 
	\[B:= \begin{pmatrix}
\frac{1}{2} & 0 & 0 & 0
\\
0 & 1 & 0 & 0
\\
0 & 0 & 1 & 0
\\
0 & 0 & 0 & 1
\end{pmatrix}.
\]
Then 
$$N_m=\sigma\circ(0 \;\;\ 1 \;\;\ 1 \;\; -1)\circ \sigma\circ B\circ \sigma\circ A\circ...\circ \sigma\circ B\circ \sigma\circ A\circ\sigma\circ B\circ \sigma\circ A,$$
where each of the mutually succeeding matrices $A$ and $B$ appears in the above representation $m$ times.
Using parameters $\{0,\pm \frac{1}{2}, \pm 1\}$, for a given input $(1, x, y)$ the first two layers of the network Mult$_m$ compute the vector
$$\bigg(1, \frac{1}{4}, T_+\bigg(\frac{x-y+1}{2}\bigg), \bigg(\frac{x+y}{2}\bigg)_+, T_-^1\bigg(\frac{x-y+1}{2}\bigg), \frac{1}{4}, T_+\bigg(\frac{x+y}{2}\bigg),  \frac{1}{4}, T_-^1\bigg(\frac{x+y}{2}\bigg)\bigg)$$ (note that as in our construction we omit shift vectors, throughout the whole construction we will keep the first coordinate equal to $1$). We then apply the network $N_m$ to the first and last four coordinates of the above vector that follow the first coordinate $1$. We thus obtain a network with $2m+2$ hidden layers and of width $9$ that computes 
\begin{equation}\label{O}
(1, x, y)\mapsto \bigg(1, \sum_{k=1}^{m+1}R^k\bigg(\frac{x-y+1}{2}\bigg)+\frac{x+y}{2}, \sum_{k=1}^{m+1}R^k\bigg(\frac{x+y}{2}\bigg)+\frac{1}{4}\bigg).\end{equation}
Finally, the last two layers of Mult$_m$ compute $(1, u, v)\mapsto(1-(1-(u-v))_+)_+$ applied to the vector obtained in \eqref{O} (note that this computation only requires parameters $0$ and $\pm 1$). We thus get a network Mult$_m$ computing \eqref{mul} and the inequality \eqref{in} follows by combining \eqref{x} and \eqref{g}.
\end{proof}

\begin{lemma}\label{Multr}
	For any positive integer $m$, there exists a network 
	\emph{Mult$^r_m\in\mathcal{\widetilde{F}}(L, \textbf{p})$}, with $L=(2m+5)\lceil \log_2r \rceil,$ $p_0=r+1,$ $p_{L+1}=1$ and $|\emph{\textbf{p}}|_\infty\leq 9r,$ such that
	
	$$|\emph{Mult}^r_m(1, \emph{\textbf{x}})-\prod_{i=1}^{r}x_i|\leq r^22^{-m}, \quad \textit{for all} \; \emph{\textbf{x}}=(x_1\;...\;x_r)\in[0,1]^r.$$
	
\end{lemma}
\begin{proof}

In order to approximate the product $\prod_{i=1}^{r}x_i$ we first pair the neighbouring entries to get the triples $(1, x_k, x_{k+1})$, and apply the previous lemma to each of those triples to obtain the values $ \textrm{{Mult}}_m(1, x_k,x_{k+1})$. We repeat this procedure $q:=\lceil\log_2r\rceil$ times, until there is only one entry left. As pairing the entries requires only parameters $0$ and $1$, then it follows from the previous lemma that the entries of the constructed network are in $\{0,\pm \frac{1}{2}, \pm 1\}$. Using Lemma \ref{xy} and applying the inequality $|\textrm{{Mult}}_m(1, x,y)-tz|\leq2^{-m}+|x-z|+|y-t|$, $x,y,z,t\in[0,1]$, $q$ times we get $|\textrm{Mult}^r_m(1, {\textbf{x}})-\prod_{i=1}^{r}x_i|\leq3^{q-1}2^{-m}\leq r^22^{-m}$.
\end{proof}
For $\gamma>0$ let $C_{d,\gamma}$ denote the number of $d$-dimensional monomials $\textbf{x}^{\boldsymbol{\alpha}}$ with degree $|\boldsymbol{\alpha}|<\gamma$. Note that $C_{d,\gamma}<(\gamma+1)^d$. From Lemma \ref{Multr} it follows that using weights $\{0,\pm \frac{1}{2}, \pm 1\},$ we can simultaneously approximate monomials up to degree $\gamma$ (see also \cite{SH1}, Lemma A.4):
\begin{lemma}\label{Mon}
	There exists a network \emph{Mon}$_{m,\gamma}^d\in\mathcal{\widetilde{F}}(L,\emph{\textbf{p}})$ with $L\leq (2m+5)\lceil \log_2(\gamma\lor 1) \rceil+1,$ $p_0=d+1$, $p_{L+1}=C_{d,\gamma}$ and $|\emph{\textbf{p}}|_\infty\leq 9\lceil\gamma\rceil C_{d,\gamma}$ such that
	$$\bigg|\emph{Mon}_{m,\gamma}^d(1, \emph{\textbf{x}})-(\emph{\textbf{x}}^{\boldsymbol{\alpha}})_{|{\boldsymbol{\alpha}}|<\gamma}\bigg|_\infty\leq \gamma^22^{-m}, \quad \emph{\textbf{x}}\in[0,1]^d.$$
\end{lemma}
We now present the final stage of the approximation, that is, the local approximation of Taylor polynomials of $f$.

\textit{Proof of Theorem \ref{main}.} For a given $N$ let $\tilde{N}\geq N$ be the smallest integer with $\tilde{N}=(2^\nu+1)^d$ for some $\nu\in\mathbb{N}$. Note that $\tilde{N}/2^d\leq N\leq \tilde{N}.$ We are going to apply Theorem \ref{or} with $N$ in the condition of that theorem replaced by $\tilde{N}$.
For $\textbf{a}\in[0,1]^d$ let
\begin{equation}\label{part}
P_\textbf{a}^\beta f(\textbf{x})=\sum\limits_{0\leq|\boldsymbol{\alpha}|<\beta}(\partial^{\boldsymbol{\alpha}}f)(\textbf{a})\frac{(\textbf{x}-\textbf{a})^{\boldsymbol{\alpha}}}{\boldsymbol{\alpha}!}:=\sum\limits_{0\leq|\boldsymbol{\gamma}|<\beta}c_{{\textbf{a},\boldsymbol{\gamma}}}\textbf{x}^{\boldsymbol{\gamma}}
\end{equation}
be the partial sum of Taylor series of $f$ around $\textbf{a}$. Choose $M$ to be the largest integer such that $(M+1)^d\leq \tilde{N},$ that is, $M=2^\nu$, and consider the set of $(M+1)^d$ grid points $\textbf{D}(M):=\{\textbf{x}_{\boldsymbol{\ell}}=(\ell_j/M)_{j=1,...,d}:\boldsymbol{\ell}=(\ell_1,...,\ell_d)\in\{0,1,...,M\}^d\}$. Denoting $\textbf{x}_{\boldsymbol{\ell}}=(x_1^{\boldsymbol{\ell}},...,x_d^{\boldsymbol{\ell}})$ it is shown in \cite{SH1}, Lemma B.1, that 
\begin{equation}\label{a}
\|P^\beta f-f\|_{L^\infty[0,1]^d}\leq KM^{-\beta},
\end{equation}
where
$$P^\beta f(\textbf{x})=\sum\limits_{\textbf{x}_{\boldsymbol{\ell}}\in\textbf{D}(M)}P_{\textbf{x}_{\boldsymbol{\ell}}}^\beta f(\textbf{x})\prod_{j=1}^{d}(1-M|x_j-x_j^{\boldsymbol{\ell}}|)_+.$$
As in our construction we only use parameters $\{0,\pm \frac{1}{2}, \pm 1, 2\}$, we need to modify the coefficients given in \eqref{part} to make them implementable by those parameters. Denote $B:=\lfloor 2Ke^d\rfloor$ and let $b\in\mathbb{N}$ be the smallest integer with $2^b\geq BM^\beta(\beta+1)^d$. As $|c_{{\textbf{a},\boldsymbol{\gamma}}}|<B$ (\cite{SH1}, eq. 34), then for each $c_{{\textbf{a},\boldsymbol{\gamma}}}$ there is an integer $k\in[-2^b, 2^b]$ with $c_{{\textbf{a},\boldsymbol{\gamma}}}\in[\frac{k}{2^b}B,\frac{k+1}{2^b}B)$. Denote then $\tilde{c}_{{\textbf{a},\boldsymbol{\gamma}}}=\frac{k}{2^b}B$ and define
$$\tilde{P}_\textbf{a}^\beta f(\textbf{x})=\sum\limits_{0\leq|\boldsymbol{\gamma}|<\beta}\tilde{c}_{{\textbf{a},\boldsymbol{\gamma}}}\textbf{x}^{\boldsymbol{\gamma}}.$$
As the number of monomials of degree up to $\beta$ is bounded by $(\beta+1)^d$, then 
$$\|P_\textbf{a}^\beta f-\tilde{P}_\textbf{a}^\beta f\|_{L^\infty[0,1]^d}\leq(\beta+1)^d\frac{B}{2^b}\leq M^{-\beta}.$$
Also, as 
$$\sum\limits_{\textbf{x}_{\boldsymbol{\ell}}\in\textbf{D}(M)} \prod_{j=1}^{d}(1-M|x_j-x_j^{\boldsymbol{\ell}}|)_+=1,$$
then 
$$\|P^\beta f(\textbf{x})-\sum\limits_{\textbf{x}_{\boldsymbol{\ell}}\in\textbf{D}(M)}\tilde{P}_{\textbf{x}_{\boldsymbol{\ell}}}^\beta f(\textbf{x})\prod_{j=1}^{d}(1-M|x_j-x_j^{\boldsymbol{\ell}}|)_+\|_{L^\infty[0,1]^d}\leq M^{-\beta}.$$
Thus, defining 
$$\tilde{P}^\beta f(\textbf{x})=\sum\limits_{\textbf{x}_{\boldsymbol{\ell}}\in\textbf{D}(M)}\tilde{P}_{\textbf{x}_{\boldsymbol{\ell}}}^\beta f(\textbf{x})\prod_{j=1}^{d}(1-M|x_j-x_j^{\boldsymbol{\ell}}|)_+,$$
we get that 
\begin{equation}\label{tilde}
\|\tilde{P}^\beta f-f\|_{L^\infty[0,1]^d}\leq (K+1)M^{-\beta}.
\end{equation}
In the proof of Theorem \ref{or}, the neural network approximation of the function $(x_1,...,x_r)\mapsto\prod_{j=1}^{r}x_j$ is first constructed followed by approximation of monomials of degree up to $\beta$. The result then follows by approximating the function $P^\beta f(\textbf{x})$ and applying \eqref{a}. In the latter approximation the set of parameters not belonging to $\{0,\pm \frac{1}{2}, \pm 1\}$ consists of:
\begin{itemize}
 
 \item shift coordinates $j/M,$ $j=1,...,M-1,$ (the grid points);
 \item at most $(\beta(M+1))^d$ weight matrix entries of the form $c_{{\textbf{x}_{\boldsymbol{\ell}},\boldsymbol{\gamma}}}/B$, where $c_{{\textbf{x}_{\boldsymbol{\ell}},\boldsymbol{\gamma}}}$ are coefficients of the polynomial $P^\beta_{\textbf{x}_{\boldsymbol{\ell}}}f(\textbf{x}),$ $\textbf{x}_{\boldsymbol{\ell}}\in\textbf{D}(M), |\boldsymbol{\gamma}|<\beta$;
 \item a shift coordinate $1/(2M^d)$ (used to scale the output entries).
\end{itemize}
Note that the above list gives at most $D:=M+(\beta(M+1))^d$ different parameters. Taking into account \eqref{tilde} we can use $\tilde{P}^\beta f$ instead of ${P}^\beta f$ to approximate $f$. Thus, we can replace the entries $c_{{\textbf{x}_{\boldsymbol{\ell}},\boldsymbol{\gamma}}}/B$ by the entries $\tilde{c}_{{\textbf{x}_{\boldsymbol{\ell}},\boldsymbol{\gamma}}}/B=\frac{k}{2^b},$ where $k$ is some integer from $[-2^b, 2^b]$. Also, as $M=2^\nu$, then denoting $\Delta=\max\{\nu d+1; b\}$ we need to obtain $D$ parameters from the set $\mathcal{S}=\{\frac{k}{2^\Delta},k\in\mathbb{Z}\cap (0, 2^\Delta]\}.$
As any natural number can be represented as a sum of powers of 2, then for any $y_1,...,y_D\in\mathbb{Z}\cap(0, 2^\Delta]$ we can compute 
$$(1, x_1,...,x_d)\mapsto(1, x_1,...,x_d, y_1,...,y_D)$$ with parameters from $\{0, 1, 2\}$ using at most ${\Delta}$ hidden layers. The number of active parameters required for this computation is bounded by $(1+d+D+\Delta){\Delta}$. Hence, for any $z_1,...,z_D\in\mathcal{S}$, we can compute $$(1, x_1,...,x_d)\mapsto(1, x_1,...,x_d, z_1,...,z_D)$$
with $2\Delta$ hidden layers and $2(1+d+D+\Delta){\Delta}$ active parameters.
Applying Theorem \ref{or} we get the existence of a network $\tilde{f}\in\mathcal{\widetilde{F}}(\tilde{L},\normalfont\tilde{\textbf{p}},\tilde{s})$ with the desired architecture and sparsity.
\qed
\section*{Acknowledgement} The author has been supported by the NWO Vidi grant: ``\textit{Statistical foundation for multilayer neural networks}''. The author would like to thank Johannes Schmidt-Hieber and Hailin Sang for useful comments and suggestions.

\end{document}